\documentclass[sigconf]{acmart}

\usepackage{amsmath}
\usepackage{booktabs}
\usepackage{graphicx}
\usepackage{mathtools}

\renewcommand\footnotetextcopyrightpermission[1]{}
\setcopyright{none}
\acmConference[CONSEQUENCES '26]{The 5th Workshop on Causality, Counterfactuals and Sequential Decision-Making for Recommender Systems}{October 2, 2026}{Minneapolis, MN, USA}
\acmBooktitle{The 5th Workshop on Causality, Counterfactuals and Sequential Decision-Making for Recommender Systems (CONSEQUENCES '26), October 2, 2026, Minneapolis, MN, USA}
\acmYear{2026}
\copyrightyear{2026}
\acmDOI{}
\acmISBN{}

\newtheorem{proposition}{Proposition}
\newtheorem{theorem}{Theorem}

\newcommand{\E}{\mathbb E}
\newcommand{\Var}{\mathbb V}

\title{Adaptive Doubly Robust Off-Policy Evaluation for Ranking Policies under Diverse User Behavior}
\hypersetup{pdfauthor={Kosuke Iguchi, Ren Kishimoto}}

\author{Kosuke Iguchi}
\affiliation{%
  \institution{Institute of Science Tokyo}
  \city{Tokyo}
  \country{Japan}
}
\email{iguchi.k.3444@m.isct.ac.jp}

\author{Ren Kishimoto}
\affiliation{%
  \institution{Institute of Science Tokyo}
  \city{Tokyo}
  \country{Japan}
}
\email{kishimoto.r.ab@m.titech.ac.jp}

\begin{document}

\begin{abstract}
Off-policy evaluation (OPE) of ranking policies is challenging because selecting and ordering multiple items from a candidate set makes the number of possible rankings grow combinatorially with the number of candidates and the ranking length.
Consequently, Inverse Propensity Scoring (IPS), whose importance weight is the full-ranking probability ratio under the evaluation and logging policies, can have excessive variance.
Independent IPS (IIPS) and Reward Interaction IPS (RIPS) reduce variance by imposing fixed assumptions on how users browse rankings, but may introduce bias when those assumptions mismatch actual behavior.
Adaptive Inverse Propensity Scoring (AIPS) addresses this trade-off by adaptively marginalizing importance weights over the actions that affect each position-wise reward.
It attains minimum variance within a class of unbiased IPS-based estimators when the true user behavior model is observed.
However, its estimation accuracy may still degrade for longer rankings, and AIPS does not use a reward model for residual correction.
We propose Adaptive Doubly Robust (ADR), which combines adaptive importance weighting with reward regression through a control-variate correction.
We establish its unbiasedness when the true user behavior model is observed and characterize a sufficient condition under which it reduces variance relative to AIPS.
Across synthetic experiments with 10,000 simulations per condition, ADR improves mean squared error over AIPS and conventional ranking OPE estimators across a range of logged-data sizes and ranking lengths.
\end{abstract}


\maketitle

\section{Introduction}

Ranking interfaces, which present users with an ordered list of multiple items,
are a common presentation format in recommendation and search systems
\cite{swaminathan2017slate,mcinerney2020rips,kiyohara2022cascade,kiyohara2023aips}.
The most direct way to evaluate a new ranking policy is to deploy it in an
online A/B test. However, online experiments can be costly and time-consuming, and
deploying a poor policy may damage user satisfaction
\cite{gilotte2018offlineab,gruson2019playlist,saito2021counterfactual}.
Off-policy evaluation (OPE) instead estimates the value of a new policy using
only logged data collected by a past policy
\cite{li2011news,dudik2011dr,saito2021counterfactual}.

OPE has been studied extensively in contextual bandits
\cite{dudik2011dr,li2011news,wang2017optimal}, but importance-weighting estimators can
suffer from high variance when the number of available actions is large
\cite{saito2022mips,saito2023offcem}.
The same challenge arises in combinatorial action spaces, where a policy
selects multiple actions simultaneously
\cite{swaminathan2017slate,kiyohara2024lips,shimizu2024opcb}.
For ranking policies in particular, the action space grows combinatorially with
the number of candidate items and the ranking length, causing the
full-ranking importance weights used by Inverse Propensity Scoring (IPS) to
have high variance
\cite{li2018clickmodels,mcinerney2020rips,kiyohara2023aips}.

To reduce this variance, estimators that exploit reward structure within a
ranking have been proposed, including Independent IPS (IIPS)
\cite{li2018clickmodels} under position-wise independence and Reward
Interaction IPS (RIPS) \cite{mcinerney2020rips} under cascade reward
interactions.
Studies of click modeling suggest that
examination and click behavior can vary across session contexts and users
\cite{chen2020contextaware,zhang2022personalized}.
Consequently, misspecifying the reward-interaction structure can introduce bias.
Adaptive Inverse Propensity Scoring (AIPS)~\cite{kiyohara2023aips},
which generalizes IPS, IIPS, and RIPS,
improves this bias--variance trade-off through adaptive
importance weighting based on context-dependent user behavior models.
When the user behavior model is observed, AIPS is unbiased under
context-dependent distributions of user behavior and achieves minimum variance
within the class of IPS-based unbiased estimators considered by Kiyohara et al.

Doubly Robust (DR) estimators use a reward model as a control
variate and apply importance weighting to the residual, which can improve the
bias--variance trade-off when the reward model is sufficiently informative
\cite{dudik2011dr,wang2017optimal,su2020shrinkage}.
For ranking policies, Cascade-DR combines this correction with a cascade
assumption shared across the population \cite{kiyohara2022cascade}.
AIPS adaptively marginalizes the importance weights but does not use a
reward model as a control variate.
We therefore propose \emph{Adaptive Doubly Robust (ADR)}, which combines the
adaptive importance weighting of AIPS with a reward-model residual correction.

Our contributions are threefold:
(1) we define the ADR estimator, combining the adaptive importance weights of AIPS with a reward-model residual correction;
(2) we establish unbiasedness when the true user behavior model is observed, characterize the variance difference relative to AIPS, and derive a sufficient condition for variance reduction; and
(3) we implement ADR by learning both models from logged data and show that it achieves lower mean squared error (MSE) than AIPS and existing estimators across a broad range of logged-data sizes and ranking lengths.

\vskip-\baselineskip
\section{Preliminaries}

This section introduces the ranking-policy OPE setting and the estimators
needed to define our approach.

\subsection{Off-Policy Evaluation of Ranking Policies}

Let $x\in\mathcal X$ be a context, $\mathcal A$ a finite action set,
$a=(a_1,\ldots,a_L)$ a ranking action of length $L$, and
$r=(r(1),\ldots,r(L))$ its position-wise reward vector.
Let $\pi_0(a\mid x)$ and $\pi_e(a\mid x)$ denote the logging and evaluation
policies.
Each logged observation $(x_i,a_i,r_i)$ is generated from the joint density
$p(x_i)\pi_0(a_i\mid x_i)p(r_i\mid x_i,a_i)$.
\mbox{For $D=\{(x_i,a_i,r_i)\}_{i=1}^n$,} the joint density is
\[
  p(D)
  =
  \prod_{i=1}^n
  p(x_i)\pi_0(a_i\mid x_i)p(r_i\mid x_i,a_i)
\]
Let
\[
  q_l(x,a)
  :=
  \mathbb E_{p(r(l)\mid x,a)}[r(l)]
\]
denote the expected reward at position $l$.
Our goal is to estimate the value of an evaluation policy $\pi_e$,
\begin{equation}
  V(\pi_e)
  =
  \sum_{l=1}^L\alpha_l
  \mathbb E_{p(x)\pi_e(a\mid x)}
  [q_l(x,a)]
  \label{eq:value}
\end{equation}
where $\alpha_l\geq0$ is a position weight. Setting
$\alpha_l=1/\log_2(l+1)$ makes the policy value identical to Discounted
Cumulative Gain (DCG) \cite{kiyohara2023aips}.
Let $\widehat V(\pi_e;D)$ denote an estimator of $V(\pi_e)$ constructed from
logged data $D$. Its accuracy is quantified by the mean squared error
\begin{align}
  \operatorname{MSE}(\widehat V)
  &:=
  \mathbb E_{p(D)}
  \left[
    \{V(\pi_e)-\widehat V(\pi_e;D)\}^2
  \right]\notag\\
  &=
  \operatorname{Bias}(\widehat V)^2
  +
  \operatorname{Var}(\widehat V)
  \label{eq:mse}
\end{align}
Here,
\begin{align*}
  \operatorname{Bias}(\widehat V)
  &:=
  \mathbb E_{p(D)}[\widehat V(\pi_e;D)]
  -
  V(\pi_e)\\
  \operatorname{Var}(\widehat V)
  &:=
  \mathbb E_{p(D)}
  \left[
    \left\{
      \widehat V(\pi_e;D)
      -
      \mathbb E_{p(D)}[\widehat V(\pi_e;D)]
    \right\}^2
  \right]
\end{align*}
This is the standard bias--variance decomposition used in OPE
\cite{dudik2011dr,kiyohara2023aips}.
We assume common support: for every ranking assigned positive probability by
the evaluation policy, the logging policy also assigns positive probability,
$\pi_e(a\mid x)>0\Rightarrow\pi_0(a\mid x)>0$.

A standard estimator in this setting is Inverse Propensity Scoring (IPS),
which uses the importance weight
\[
  w_{\mathrm{IPS}}(x,a)
  :=
  \frac{\pi_e(a\mid x)}{\pi_0(a\mid x)}
\]
This weight reweights rankings observed under the logging policy according to
their relative probability under the evaluation policy.
The IPS estimator is
\begin{equation}
  \widehat V_{\mathrm{IPS}}(\pi_e;D)
  =
  \frac1n\sum_{i=1}^n
  w_{\mathrm{IPS}}(x_i,a_i)
  \sum_{l=1}^L\alpha_l r_i(l)
\end{equation}
IPS imposes no particular user behavior model and is unbiased under common
support. However, full-ranking importance weights can
have extremely high variance because the action space grows combinatorially with
the ranking length \cite{mcinerney2020rips,kiyohara2023aips}.

A standard Doubly Robust (DR) estimator~\cite{dudik2011dr} uses a
position-wise reward model $\widehat q_l(x,a)$. Define its evaluation-policy
expectation as
\[
  \widehat q_l(x,\pi_e)
  :=
  \mathbb E_{\pi_e(a\mid x)}
  [\widehat q_l(x,a)]
\]
Then
\begin{align}
  \widehat V_{\mathrm{DR}}(\pi_e;D)
  &=
  \frac1n\sum_{i=1}^n\sum_{l=1}^L\alpha_l
  \left[
    \widehat q_l(x_i,\pi_e)
  \right.
  \nonumber\\
  &\hspace{9mm}\left.
    +w_{\mathrm{IPS}}(x_i,a_i)
    \{r_i(l)-\widehat q_l(x_i,a_i)\}
  \right]
\end{align}
The first term provides a reward-model baseline, while the second applies
importance weighting to its residual error; this construction can reduce
variance while preserving the unbiasedness of IPS.

In ranking recommendation, several estimators improve on IPS by imposing
assumptions on user behavior.
Independent IPS (IIPS) \cite{li2018clickmodels} assumes that the reward at
position $l$ depends only on $a_l$ and uses its marginal action-choice
probability. Reward Interaction IPS (RIPS) \cite{mcinerney2020rips} adopts the
cascade assumption and uses the marginal selection probability of the ordered
items at positions 1 through $l$, $a_{1:l}=(a_1,\ldots,a_l)$.
Their position-wise importance weights are
\[
  w_{l,\mathrm{IIPS}}(x,a_l)
  =
  \frac{\pi_e(a_l\mid x)}{\pi_0(a_l\mid x)}
  \qquad
  w_{l,\mathrm{RIPS}}(x,a_{1:l})
  =
  \frac{\pi_e(a_{1:l}\mid x)}{\pi_0(a_{1:l}\mid x)}
\]
Thus, they replace the full-ranking IPS weight with a marginal probability
ratio over the relevant actions specified by the corresponding assumption.
These structural assumptions can reduce variance when correctly specified, but
applying a single assumption to the entire population can introduce bias under
heterogeneous user behavior.

\subsection{Adaptive IPS}

Avoiding the bias caused by a single fixed user-behavior assumption requires
importance weights that adapt to context.
Adaptive IPS (AIPS)~\cite{kiyohara2023aips} addresses this need through
adaptive importance weighting based on a
context-aware user behavior model.
Specifically, when estimating the policy value at position $l$, AIPS uses the
relevant set of actions that affect the reward observed at that position to define
the importance weight.
Let $c\in\{0,1\}^{L\times L}$ be an action-reward interaction matrix, where
$c_{l,j}=1$ indicates that action $a_j$ affects reward $r(l)$, and define
\[
  \Phi_l(a,c):=\{a_j\in\mathcal A\mid c_{l,j}=1\}
\]
The user behavior model follows a context-dependent distribution
$c\sim p(c\mid x)$. Conditional on $c$, the expected reward satisfies
$q_l(x,a,c)=q_l(x,\Phi_l(a,c))
:=\mathbb E[r(l)\mid x,\Phi_l(a,c)]$.
The expected reward introduced in Section~2.1 is its marginal:
\[
  q_l(x,a)
  =
  \mathbb E_{p(c\mid x)}
  [q_l(x,a,c)]
\]
The independence, cascade, and no-assumption models correspond to using
$a_l$, $a_{1:l}$, and the full ranking $a$, respectively.

Following the theoretical setting of AIPS, we assume that $c_i$ is observed
for each logged sample. Thus, $c_i$ is treated as an additional observed
component of $D$ in the AIPS and ADR definitions below.
The adaptive importance weight and AIPS estimator are
\begin{align}
  w_l(x,a,c)
  &:=
  \frac{\pi_e(\Phi_l(a,c)\mid x)}
       {\pi_0(\Phi_l(a,c)\mid x)}
  \label{eq:aips-weight}\\
  \widehat V_{\mathrm{AIPS}}(\pi_e;D)
  &:=
  \frac1n\sum_{i=1}^n\sum_{l=1}^L
  \alpha_l w_l(x_i,a_i,c_i)r_i(l)
\end{align}
Here, $\pi(\Phi_l(a,c)\mid x)$ marginalizes the policy probability over rankings
that agree on the relevant actions.
When $c$ is observed, AIPS is unbiased under context-dependent distributions of
user behavior and achieves minimum variance within the class of IPS-based unbiased
estimators considered by Kiyohara et al.\ \cite{kiyohara2023aips}.
However, accurate OPE with AIPS can remain difficult when the number of unique
actions is large. Moreover, AIPS is an IPS-based estimator and does not use a
reward model as a control variate.

\section{Our Approach}

This section defines Adaptive Doubly Robust (ADR), which combines AIPS's
adaptive importance weighting with a reward-model residual correction, and
analyzes its unbiasedness and variance reduction.

\subsection{Estimator}

For each position $l$, let $\widehat q_l$ be a reward model evaluated at
$(x,\Phi_l(a,c))$. Define its evaluation-policy expectation as
\[
  \widehat q_l(x,c,\pi_e)
  :=
  \mathbb{E}_{\pi_e(a\mid x)}
  [\widehat q_l(x,\Phi_l(a,c))]
\]
The ADR estimator takes the form
\begin{align}
  \widehat V_{\mathrm{ADR}}(\pi_e;D)
  &=
  \frac{1}{n}\sum_{i=1}^{n}\sum_{l=1}^{L}\alpha_l
  \left[
    \widehat q_l(x_i,c_i,\pi_e)
  \right.
  \nonumber\\
  &\hspace{9mm}\left.
    +w_l(x_i,a_i,c_i)
    \{r_i(l)-\widehat q_l(x_i,\Phi_l(a_i,c_i))\}
  \right]
  \label{eq:adr}
\end{align}
The first term is the reward-model prediction under the evaluation policy, and the
second term corrects its error by applying the adaptive importance weight to the
observed residual. Following standard DR estimation \cite{dudik2011dr}, the
reward model serves as a control variate, while AIPS supplies the
adaptive importance weight for the residual correction. Setting
$\widehat q\equiv0$ recovers AIPS.

\subsection{Theoretical Properties}

Following the theoretical setting of AIPS, we treat $c_i$ as observed and
assume that the reward model $\widehat q$ is given independently of the logged
data $D$.

\begin{proposition}[Unbiasedness of ADR]\hfill\break
\label{prop:dr}
Assume that the logged observations are independent and identically
distributed, the true user behavior model $c_i$ is observed for every
observation, and common support holds.
Then, regardless of the predictive accuracy of
$\widehat q_l(x,\Phi_l(a,c))$,
\[
  \E_{p(D)}
  [\widehat V_{\mathrm{ADR}}(\pi_e;D)]
  =
  V(\pi_e)
\]
\end{proposition}

In ADR, the marginalized importance weight induced by the observed user
behavior model $c$ reweights the residual expectation under the logging policy
to match that under the evaluation policy.
The complete proof is given in Appendix~\ref{app:proof-dr}.

\begin{theorem}[Variance difference from AIPS]\hfill\break
\label{thm:var}
Assume the theoretical setting above and common support.
Let
\[
\Delta_{q,\widehat q,l}(x,a,c)
:=\widehat q_l(x,\Phi_l(a,c))-q_l(x,a,c)
\]
Then
\begin{align}
&\Var_{p(D)}(\widehat V_{\mathrm{AIPS}}(\pi_e;D))
-\Var_{p(D)}(\widehat V_{\mathrm{ADR}}(\pi_e;D))
\nonumber\\
&=\frac1n
\E_{p(x,c)}
\Bigg[
  \Var_{\pi_0(a\mid x)}
  \left(
    \sum_{l=1}^L
    \alpha_lw_l(x,a,c)q_l(x,a,c)
  \right)
\nonumber\\
&\hspace{18mm}-
  \Var_{\pi_0(a\mid x)}
  \left(
    \sum_{l=1}^L
    \alpha_lw_l(x,a,c)\Delta_{q,\widehat q,l}(x,a,c)
  \right)
\Bigg]
\label{eq:variance-difference}
\end{align}
\end{theorem}

This variance difference compares the variation of the weighted expected
reward with that of the weighted reward-model error.
Because multiple positions share the same ranking, the variance-reduction
condition is evaluated for the weighted sum rather than position by position.
If the latter is no larger than the former, ADR has no larger variance than AIPS.
If $\widehat q_l(x,\Phi_l(a,c))=q_l(x,a,c)$ for every $l,x,a,c$, then
$\Delta_{q,\widehat q,l}=0$ and the condition always holds.
The complete proof is given in Appendix~\ref{app:proof-var}.

\section{Experiments}

We evaluate ADR as the logged-data size and ranking length vary.

\subsection{Setup}

We use the synthetic ranking OPE environment of Kiyohara et
al.~\cite{kiyohara2023aips} and compare six estimators: IPS, IIPS, RIPS, AIPS,
ADR, and ADR (oracle). AIPS and ADR share the learned user behavior model and
self-normalized weights and differ only in whether the learned reward model is
used; the oracle instead uses the simulator's latent true expected reward. The data-size experiment
fixes $L=8$ and varies $n$ from $1{,}000$ to $32{,}000$ by powers of two,
whereas the
ranking-length experiment fixes $n=8{,}000$ and varies
$L\in\{4,6,8,10,12,14\}$. We run 10,000 simulations per condition and
evaluate MSE. Squared bias, variance, and full implementation
details are reported in Appendix~\ref{app:experiment-details}.

\begin{figure*}[t]
  \centering
  \includegraphics[width=\textwidth]{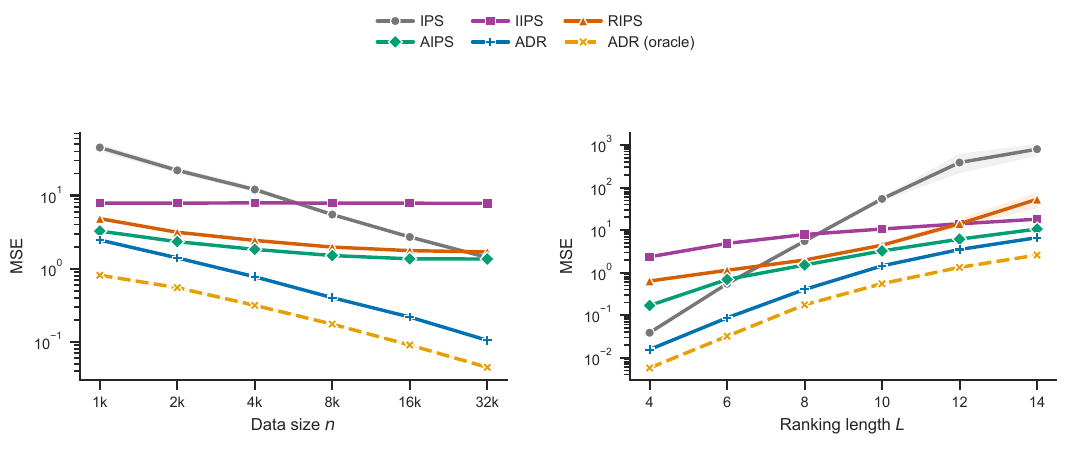}
  \caption{MSE against data size (left) and ranking length (right).}
  \Description{Two line charts compare the MSE of six off-policy estimators as data size and ranking length vary.}
  \label{fig:main-mse}
\end{figure*}

\subsection{Results and Discussion}

\paragraph{Research Question 1 (RQ1): Effect of data size on ADR}\hfill\break
IPS is nearly unbiased but has high variance because it weights the full
ranking. IIPS has low variance but substantial squared bias, while RIPS cannot
eliminate squared bias when the cascade assumption does not hold.
This agrees with Kiyohara et al.~\cite{kiyohara2023aips}: imposing a single
assumption under context-dependent heterogeneous user behavior can introduce bias.
The complete bias--variance decomposition for all estimators is reported in
Appendix~\ref{app:experiment-details}.
At $n=32{,}000$, IPS is nearly unbiased with squared bias
$1.04\times10^{-4}$, but its variance of $1.423$ dominates its error.
IIPS reduces variance to $0.00077$ but retains squared bias of $7.836$.
RIPS reduces variance to $0.0975$ but incurs squared bias of $1.597$
(Table~\ref{tab:representative-decomposition}).

Figure~\ref{fig:main-mse} (left) shows that ADR achieves the lowest MSE among
all non-oracle estimators at every data size. Its MSE decreases from $2.465$ at
$n=1{,}000$ to $0.105$ at $n=32{,}000$ and remains below AIPS throughout.

\begin{table*}[t]
  \caption{Bias--variance decomposition of AIPS and ADR across all conditions.}
  \label{tab:aips-decomposition}
  \centering
  \fontsize{10pt}{12pt}\selectfont
  \begin{minipage}[t]{0.48\textwidth}
    \centering
    \textbf{(a) Data size ($L=8$)}\par\smallskip
    \setlength{\tabcolsep}{3pt}
    \begin{tabular}{@{}lrrrr@{}}
      \toprule
      $n$ & \multicolumn{2}{c}{Bias$^2$} & \multicolumn{2}{c}{Variance} \\
      \cmidrule(lr){2-3}\cmidrule(l){4-5}
      & AIPS & ADR & AIPS & ADR \\
      \midrule
      1,000  & 0.986 & \textbf{0.0959}  & \textbf{2.279} & 2.369 \\
      2,000  & 0.898 & \textbf{0.0307}  & 1.444 & \textbf{1.375} \\
      4,000  & 0.987 & \textbf{0.0232}  & 0.850 & \textbf{0.754} \\
      8,000  & 1.053 & \textbf{0.0212}  & 0.459 & \textbf{0.381} \\
      16,000 & 1.099 & \textbf{0.0131}  & 0.266 & \textbf{0.206} \\
      32,000 & 1.205 & \textbf{0.00554} & 0.151 & \textbf{0.0998} \\
      \bottomrule
    \end{tabular}
  \end{minipage}\hfill
  \begin{minipage}[t]{0.48\textwidth}
    \centering
    \textbf{(b) Ranking length ($n=8{,}000$)}\par\smallskip
    \setlength{\tabcolsep}{3pt}
    \begin{tabular}{@{}lrrrr@{}}
      \toprule
      $L$ & \multicolumn{2}{c}{Bias$^2$} & \multicolumn{2}{c}{Variance} \\
      \cmidrule(lr){2-3}\cmidrule(l){4-5}
      & AIPS & ADR & AIPS & ADR \\
      \midrule
      4  & 0.148 & $\mathbf{8.69{\times}10^{-6}}$ & 0.0211 & \textbf{0.0156} \\
      6  & 0.578 & \textbf{0.00105} & 0.112 & \textbf{0.0856} \\
      8  & 1.051 & \textbf{0.0214}  & 0.459 & \textbf{0.381} \\
      10 & 1.726 & \textbf{0.129}   & 1.513 & \textbf{1.315} \\
      12 & 2.572 & \textbf{0.426}   & 3.543 & \textbf{3.073} \\
      14 & 4.038 & \textbf{1.024}   & 6.543 & \textbf{5.620} \\
      \bottomrule
    \end{tabular}
  \end{minipage}
\end{table*}

The bias--variance decomposition shows that ADR substantially reduces squared
bias at every data size. Its variance is slightly larger than that of AIPS at
$n=1{,}000$ ($2.369$ versus $2.279$), but is lower for $n\geq2{,}000$. At
$n=32{,}000$, ADR reduces squared bias from $1.205$ to $0.00554$ and variance
from $0.151$ to $0.100$ (Table~\ref{tab:aips-decomposition}). Thus, the MSE
improvement in the data-size experiment reflects reductions in both squared
bias and variance.

\paragraph{RQ2: Effect of ranking length on ADR}\hfill\break
For longer rankings, IPS variance grows rapidly, IIPS accumulates squared bias,
and both squared bias and variance increase for RIPS. At $L=14$, IPS variance
reaches $796.278$.
IIPS retains low variance ($0.00558$) but has squared bias $18.035$, while
RIPS has variance $50.291$ and squared bias $2.583$
(Table~\ref{tab:representative-decomposition}).

AIPS achieves a better bias--variance trade-off than estimators based on these
fixed user-behavior assumptions. However, both its squared bias and variance
increase with the ranking length.

In contrast, Figure~\ref{fig:main-mse} (right) shows that although the MSE of
every non-oracle estimator increases with the ranking length, ADR achieves the
lowest MSE for every $L$. Its MSE rises from $0.0156$ at $L=4$ to $6.645$ at
$L=14$, compared with $0.169$ and $10.581$ for AIPS. The corresponding
improvement over AIPS is $90.8\%$ at $L=4$ and $37.2\%$ at $L=14$.

In a direct comparison with AIPS, ADR reduces squared bias from $4.038$ to
$1.024$ and variance from $6.543$ to $5.620$ at $L=14$, improving both error
components (Table~\ref{tab:aips-decomposition}). Thus, ADR maintains more
stable estimation accuracy than the existing estimators as ranking length
increases.

In both experiments, ADR (oracle) achieves lower MSE than ADR with a learned
reward model. This gap suggests that improving reward prediction can further
improve ADR's estimation accuracy.

\begin{table}[H]
  \caption{Representative bias--variance decomposition.}
  \label{tab:representative-decomposition}
  \centering
  \fontsize{10pt}{10.5pt}\selectfont
  \begin{minipage}[t]{0.48\columnwidth}
    \centering
    \textbf{$n=32{,}000$}\par\smallskip
    \setlength{\tabcolsep}{2pt}
    \begin{tabular}{@{}lrr@{}}
      \toprule
      Est. & Bias$^2$ & Var. \\
      \midrule
      IPS  & 0.00010 & 1.423 \\
      IIPS & 7.836 & 0.00077 \\
      RIPS & 1.597 & 0.0975 \\
      \bottomrule
    \end{tabular}
  \end{minipage}\hfill
  \begin{minipage}[t]{0.48\columnwidth}
    \centering
    \textbf{$L=14$}\par\smallskip
    \setlength{\tabcolsep}{2pt}
    \begin{tabular}{@{}lrr@{}}
      \toprule
      Est. & Bias$^2$ & Var. \\
      \midrule
      IPS  & 0.690 & 796.278 \\
      IIPS & 18.035 & 0.00558 \\
      RIPS & 2.583 & 50.291 \\
      \bottomrule
    \end{tabular}
  \end{minipage}
\end{table}

\section{Conclusion}

ADR combines adaptive importance weighting with a reward-model residual correction for ranking OPE.
It is unbiased when the true user behavior model is observed, and our analysis
gives a sufficient condition for variance reduction; experiments show lower
MSE than existing estimators. Because they use only synthetic data, validation
on real logs is needed. Future work includes combining ADR with action
abstractions or embeddings \cite{saito2022mips,kiyohara2024lips} and extending
it to off-policy learning with generalization or regret guarantees
\cite{dudik2011dr,shimizu2024opcb}.

\clearpage
\bibliographystyle{ACM-Reference-Format}
\bibliography{references}

\clearpage
\appendix
\onecolumn
\raggedbottom
\section{Theoretical Results and Proofs}

\subsection{Proof Setup}

The joint distribution of the logged data is
\begin{align}
  p(D)
  =
  \prod_{i=1}^n
  p(x_i)p(c_i\mid x_i)
  \pi_0(a_i\mid x_i)
  p(r_i\mid x_i,a_i,c_i)
  \label{eq:app-data}
\end{align}
For one logged observation, define
\begin{align*}
  p_D(x,c,a,r)
  &:=
  p(x)p(c\mid x)\pi_0(a\mid x)p(r\mid x,a,c)\\
  p_D(x,c,a)
  &:=
  p(x)p(c\mid x)\pi_0(a\mid x)
\end{align*}
We also write
\[
  q_l(x,c,\pi)
  :=
  \E_{\pi(a\mid x)}[q_l(x,a,c)]
  \qquad
  q(x,c,\pi)
  :=
  \sum_{l=1}^L\alpha_lq_l(x,c,\pi)
\]
Then
\[
  V(\pi_e)
  =
  \E_{p(x)}\E_{p(c\mid x)}
  [q(x,c,\pi_e)]
\]

\subsection{Proof of Proposition~\ref{prop:dr}}
\label{app:proof-dr}

\begin{proof}
Substituting Eq.~\eqref{eq:adr} into its expectation gives
\begin{align}
&\E_{p(D)}
[\widehat V_{\mathrm{ADR}}(\pi_e;D)]
\nonumber\\
&=
\E_{p(D)}
\left[
  \frac1n\sum_{i=1}^n\sum_{l=1}^L\alpha_l
  \left\{
    \widehat q_l(x_i,c_i,\pi_e)
    +w_l(x_i,a_i,c_i)
    \left(r_i(l)-\widehat q_l(x_i,\Phi_l(a_i,c_i))\right)
  \right\}
\right]
\label{eq:app-adr-start}
\end{align}
By linearity and the identical distribution of the logged observations,
\begin{align}
&\E_{p(D)}
[\widehat V_{\mathrm{ADR}}(\pi_e;D)]
\nonumber\\
&=
\sum_{l=1}^L\alpha_l
\E_{p_D(x,c,a,r)}
\left[
  \widehat q_l(x,c,\pi_e)
  +w_l(x,a,c)
  \left(r(l)-\widehat q_l(x,\Phi_l(a,c))\right)
\right]
\label{eq:app-adr-one-sample}
\end{align}
By the definition of the position-wise expected reward under $c$,
\[
  \E_{p(r(l)\mid x,a,c)}[r(l)]
  =q_l(x,a,c)
  =q_l(x,\Phi_l(a,c))
\]
Therefore, Eq.~\eqref{eq:app-adr-one-sample} becomes
\begin{align}
&\E_{p(D)}
[\widehat V_{\mathrm{ADR}}(\pi_e;D)]
\nonumber\\
&=
\sum_{l=1}^L\alpha_l
\E_{p(x,c)}
\Bigg[
  \widehat q_l(x,c,\pi_e)
  +\E_{\pi_0(a\mid x)}
  \left[
    w_l(x,a,c)
    \left\{
      q_l(x,\Phi_l(a,c))
      -\widehat q_l(x,\Phi_l(a,c))
    \right\}
  \right]
\Bigg]
\label{eq:app-adr-common}
\end{align}
Expand the residual correction for fixed $x,c,l$, and let $\phi$ range over
the possible values of $\Phi_l(a,c)$:
\begin{align}
&\E_{\pi_0(a\mid x)}
\left[
  w_l(x,a,c)
  \left\{q_l(x,\Phi_l(a,c))-\widehat q_l(x,\Phi_l(a,c))\right\}
\right]
\nonumber\\
&=
\sum_{\phi}
\frac{\pi_e(\phi\mid x)}{\pi_0(\phi\mid x)}
\left\{q_l(x,\phi)-\widehat q_l(x,\phi)\right\}
\sum_{a:\,\Phi_l(a,c)=\phi}\pi_0(a\mid x)
\nonumber\\
&=
\sum_{\phi}\pi_e(\phi\mid x)
\left\{q_l(x,\phi)-\widehat q_l(x,\phi)\right\}
\nonumber\\
&=
\E_{\pi_e(a\mid x)}
\left[
  q_l(x,\Phi_l(a,c))-\widehat q_l(x,\Phi_l(a,c))
\right]
\label{eq:app-adr-marginalization}
\end{align}
The second equality uses
\[
  \pi_0(\phi\mid x)
  =
  \sum_{a:\,\Phi_l(a,c)=\phi}\pi_0(a\mid x)
\]
Substituting Eq.~\eqref{eq:app-adr-marginalization} into
Eq.~\eqref{eq:app-adr-common} yields
\begin{align*}
&\E_{p(D)}[\widehat V_{\mathrm{ADR}}(\pi_e;D)]
\\
&=
\sum_{l=1}^L\alpha_l\E_{p(x,c)}
\Bigg[
  \widehat q_l(x,c,\pi_e)
  +\E_{\pi_e(a\mid x)}
  \left[
    q_l(x,\Phi_l(a,c))-\widehat q_l(x,\Phi_l(a,c))
  \right]
\Bigg]
\\
&=
\sum_{l=1}^L\alpha_l\E_{p(x,c)}
\E_{\pi_e(a\mid x)}[q_l(x,\Phi_l(a,c))]
\\
&=V(\pi_e)
\end{align*}
The predicted reward in the baseline cancels the prediction in the residual;
therefore, unbiasedness does not depend on the predictive accuracy of
$\widehat q_l$.
\end{proof}

\subsection{Proof of the Variance Difference}
\label{app:proof-var}

\begin{proof}
We first restate the AIPS estimator:
\begin{align}
  \widehat V_{\mathrm{AIPS}}(\pi_e;D)
  &:=
  \frac1n\sum_{i=1}^n\sum_{l=1}^L
  \alpha_lw_l(x_i,a_i,c_i)r_i(l)
  \label{eq:app-aips-definition}
\end{align}
By this definition and the i.i.d.~property of the logged observations,
\begin{equation}
\Var_{p(D)}(\widehat V_{\mathrm{AIPS}}(\pi_e;D))
=
\frac1n
\Var_{p_D(x,c,a,r)}
\left[
  \sum_{l=1}^L\alpha_lw_l(x,a,c)r(l)
\right]
\label{eq:app-aips-start}
\end{equation}
By the law of total variance,
\begin{align}
&\Var_{p(D)}
(\widehat V_{\mathrm{AIPS}}(\pi_e;D))
\nonumber\\
&=
\frac1n
\E_{p_D(x,c,a)}
\left[
  \Var_{p(r\mid x,a,c)}
  \left(
    \sum_{l=1}^L\alpha_lw_lr(l)
  \right)
\right]
\nonumber\\
&\quad+
\frac1n
\E_{p(x,c)}
\left[
  \Var_{\pi_0(a\mid x)}
  \left(
    \sum_{l=1}^L\alpha_lw_lq_l
  \right)
\right]
\nonumber\\
&\quad+
\frac1n
\Var_{p(x,c)}
\left[
  \E_{\pi_0(a\mid x)}
  \left[
    \sum_{l=1}^L\alpha_lw_lq_l
  \right]
\right]
\label{eq:app-aips-total}
\end{align}
Here and below, $w_l,q_l$, and $\Delta_{q,\widehat q,l}$ without explicit
arguments are evaluated at $(x,a,c)$.
By the definition of the marginalized weight,
\[
  \E_{\pi_0(a\mid x)}
  \left[
    \sum_{l=1}^L\alpha_lw_lq_l
  \right]
  =
q(x,c,\pi_e)
\]
Hence,
\begin{align}
&\Var_{p(D)}
(\widehat V_{\mathrm{AIPS}}(\pi_e;D))
\nonumber\\
&=
\frac1n
\E_{p_D(x,c,a)}
\left[
  \Var_{p(r\mid x,a,c)}
  \left(
    \sum_{l=1}^L\alpha_lw_lr(l)
  \right)
\right]
\nonumber\\
&\quad+
\frac1n
\E_{p(x,c)}
\left[
  \Var_{\pi_0(a\mid x)}
  \left(
    \sum_{l=1}^L\alpha_lw_lq_l
  \right)
\right]
\nonumber\\
&\quad+
\frac1n
\Var_{p(x,c)}[q(x,c,\pi_e)]
\label{eq:app-aips-final}
\end{align}

For ADR, the law of total variance gives
\begin{align}
&\Var_{p(D)}
(\widehat V_{\mathrm{ADR}}(\pi_e;D))
\nonumber\\
&=
\frac1n
\Var_{p_D(x,c,a,r)}
\Bigg[
  \sum_{l=1}^L\alpha_l
  \Big\{
    \widehat q_l(x,c,\pi_e)
\nonumber\\
&\hspace{28mm}
    +w_l(x,a,c)
    \bigl(r(l)-\widehat q_l(x,\Phi_l(a,c))\bigr)
  \Big\}
\Bigg]
\nonumber\\
&=
\frac1n
\E_{p_D(x,c,a)}
\left[
  \Var_{p(r\mid x,a,c)}
  \left(
    \sum_{l=1}^L\alpha_lw_lr(l)
  \right)
\right]
\nonumber\\
&\quad+
\frac1n
\E_{p(x,c)}
\left[
  \Var_{\pi_0(a\mid x)}
  \left(
    \sum_{l=1}^L\alpha_lw_l
    \Delta_{q,\widehat q,l}
  \right)
\right]
\nonumber\\
&\quad+
\frac1n
\Var_{p(x,c)}[q(x,c,\pi_e)]
\label{eq:app-adr-final}
\end{align}
Subtracting Eq.~\eqref{eq:app-adr-final} from
Eq.~\eqref{eq:app-aips-final} gives
\begin{align}
&\Var_{p(D)}(\widehat V_{\mathrm{AIPS}}(\pi_e;D))
-\Var_{p(D)}(\widehat V_{\mathrm{ADR}}(\pi_e;D))
\nonumber\\
&=
\frac1n\E_{p(x,c)}\Bigg[
\Var_{\pi_0(a\mid x)}
\left(\sum_{l=1}^L\alpha_lw_lq_l\right)
\nonumber\\
&\hspace{27mm}-
\Var_{\pi_0(a\mid x)}
\left(\sum_{l=1}^L\alpha_lw_l\Delta_{q,\widehat q,l}\right)
\Bigg]
\label{eq:app-variance-difference}
\end{align}
Hence, if
\begin{align}
&\E_{p(x,c)}\left[
\Var_{\pi_0(a\mid x)}
\left(\sum_{l=1}^L\alpha_lw_l\Delta_{q,\widehat q,l}\right)
\right]
\nonumber\\
&\qquad\leq
\E_{p(x,c)}\left[
\Var_{\pi_0(a\mid x)}
\left(\sum_{l=1}^L\alpha_lw_lq_l\right)
\right]
\label{eq:app-variance-condition}
\end{align}
Then ADR has no larger variance than AIPS. This condition requires the
variation in the importance-weighted reward-model error to be no greater than
that in the importance-weighted expected reward. It compares the variance of
the weighted sum rather than each position separately because multiple
positions share the same ranking $a$, and the variance of their sum includes
cross-position covariances. Thus, the variation in the importance-weighted
expected reward that appears in AIPS is replaced in ADR by the variation in
the importance-weighted reward-model error. In particular, if
$\widehat q_l(x,\Phi_l(a,c))=q_l(x,a,c)$ for every $l,x,a,c$, then
$\Delta_{q,\widehat q,l}=0$ and ADR has no larger variance than AIPS.
\end{proof}

\section{Supplementary Experiments}
\label{app:experiments}

\subsection{Experimental Details}
\label{app:experiment-details}

We use the synthetic ranking OPE environment from the public AIPS
implementation. Each context is sampled from the five-dimensional standard
normal distribution, and three candidate items are available at each position.
The policies are factorized across positions and allow the same item to appear
at multiple positions in a ranking; rankings can therefore be longer than the
number of candidate items. Position-wise rewards are continuous and, following
the AIPS experimental setup, are generated by adding Gaussian noise with
standard deviation $0.5$ to their conditional expectations.

An action--reward interaction matrix $c$ specifies which items in a ranking
affect each position-wise reward. We use the same six candidate matrices as the
public AIPS implementation and stochastically generate the true user behavior
model $c_i$ of each sample according to its context.

The logging policy is a softmax policy based on scores produced by Open Bandit
Pipeline's \texttt{linear\_behavior\_policy\_logit}. The evaluation policy
multiplies the logging-policy scores by $-2$ and applies an $\epsilon$-greedy
rule with exploration parameter $0.3$. Both policies assign positive
probability to every item and thus satisfy common support.

The data-size experiment fixes $L=8$ and varies
\[
n\in\{1{,}000,2{,}000,4{,}000,8{,}000,16{,}000,32{,}000\}
\]
The ranking-length experiment fixes $n=8{,}000$ and varies
\[
L\in\{4,6,8,10,12,14\}
\]

For practical AIPS and ADR, we learn a context-dependent user behavior model
$\widehat c_i$ from logged data using a User Behavior Tree (UBTree). At each
position, UBTree selects from the same six candidates as the public AIPS
implementation to reduce the estimated MSE of AIPS. We use a maximum depth of
5, a minimum of 100 samples per leaf, 10 bootstrap samples, and 10 candidate
split points. The behavior model is learned by two-fold cross-fitting, so each
sample is assigned a model trained without its fold.

For each sample and position, $\widehat c_i$ identifies the positions judged
relevant to that reward. Because the policies are factorized, we calculate the
marginal propensity under each policy by multiplying the item-selection
probabilities at those positions. The ratio of the evaluation-policy and
logging-policy marginal propensities defines the adaptive importance weight
used by AIPS and ADR. In the main experiments, AIPS and ADR share the same
estimated user behavior model and self-normalized adaptive importance weights.
Appendix~\ref{app:self-normalization} defines self-normalization and examines
its effect.

For ADR, we fit a separate HistGradientBoostingRegressor for each position.
Its input comprises the context and a one-hot representation of the entire
ranking, and its target is the observed reward at that position. We use squared
error loss, maximum depth 6, minimum leaf size 20, learning rate 0.05, 150
iterations, at most 31 leaves, and $L_2$ regularization coefficient 1.0. The
reward models are also learned by two-fold cross-fitting to produce
out-of-fold reward predictions. We approximate the predicted reward under the
evaluation policy by the Monte Carlo average over 16 rankings sampled from
that policy.

We compare IPS, IIPS, RIPS, AIPS, ADR, and ADR (oracle). AIPS and ADR share the
estimated user behavior model and self-normalized weights and differ only in
whether the learned reward-model residual correction is applied. ADR (oracle)
retains the same estimated user behavior model and weights but uses the true
expected reward conditional on the $c_i$ realized inside the simulator. This
estimator is not implementable and serves as a reference point for supplying
ideal information to the reward model. Appendix~\ref{app:oracle-ablation}
compares oracle variants that use different information.

For each condition, we run $10{,}000$ simulations. Estimation errors are
computed from the estimate and the simulator's on-policy reference value in
each simulation. The 95\% confidence bands are percentile intervals obtained
by resampling the simulation runs 400 times and recomputing each metric.

\subsection{Bias--Variance Decomposition of All Estimators}
\label{app:bias-variance}

\begin{figure}[H]
  \centering
  \includegraphics[width=\textwidth]{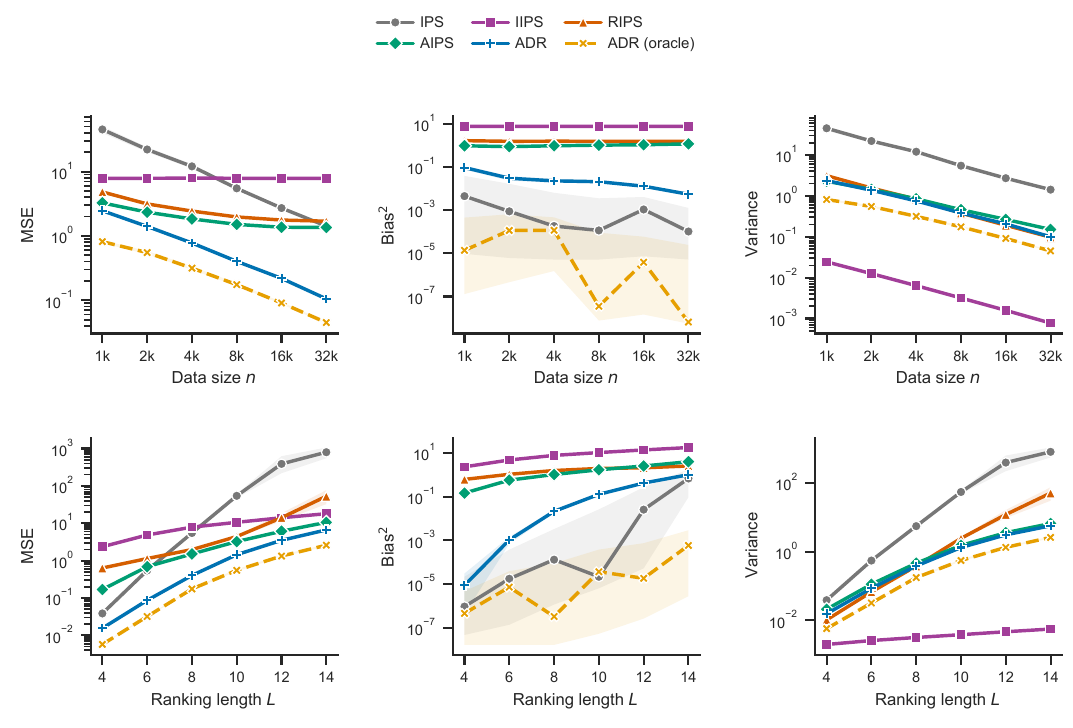}
  \caption{MSE, squared bias, and variance against data size (top) and ranking length (bottom).}
  \Description{Six line charts show MSE, squared bias, and variance for six estimators as data size and ranking length change.}
  \label{fig:appendix-matched}
\end{figure}

\begin{table}[H]
  \centering
  \caption{MSE, squared bias, and variance of all estimators in the data-size experiment.}
  \label{tab:all-estimators-data}
  \scriptsize
  \setlength{\tabcolsep}{3.5pt}
  \resizebox{0.85\textwidth}{!}{%
  \begin{tabular}{llrrrrrr}
    \toprule
    Metric & Estimator & $n=1$k & $n=2$k & $n=4$k & $n=8$k & $n=16$k & $n=32$k\\
    \midrule
    MSE & IPS & 45.160 & 22.051 & 12.053 & 5.491 & 2.713 & 1.423\\
        & IIPS & 7.841 & 7.865 & 7.923 & 7.874 & 7.843 & 7.837\\
        & RIPS & 4.826 & 3.145 & 2.435 & 1.974 & 1.766 & 1.695\\
        & AIPS & 3.266 & 2.342 & 1.837 & 1.513 & 1.365 & 1.356\\
        & ADR & 2.465 & 1.406 & 0.777 & 0.403 & 0.219 & 0.105\\
        & ADR (oracle) & \textbf{0.816} & \textbf{0.552} & \textbf{0.316} & \textbf{0.175} & \textbf{0.091} & \textbf{0.045}\\
    \midrule
    Squared bias & IPS & 0.00454 & $8.93{\times}10^{-4}$ & $1.85{\times}10^{-4}$ & $1.16{\times}10^{-4}$ & 0.00109 & $1.04{\times}10^{-4}$\\
        & IIPS & 7.817 & 7.852 & 7.917 & 7.870 & 7.841 & 7.836\\
        & RIPS & 1.702 & 1.586 & 1.608 & 1.596 & 1.575 & 1.597\\
        & AIPS & 0.986 & 0.898 & 0.987 & 1.053 & 1.099 & 1.205\\
        & ADR & 0.096 & 0.031 & 0.023 & 0.021 & 0.013 & 0.00554\\
        & ADR (oracle) & \textbf{1.39\(\times\)10\textsuperscript{-5}} & \textbf{1.15\(\times\)10\textsuperscript{-4}} & \textbf{1.17\(\times\)10\textsuperscript{-4}} & \textbf{3.51\(\times\)10\textsuperscript{-8}} & \textbf{3.86\(\times\)10\textsuperscript{-6}} & \textbf{6.47\(\times\)10\textsuperscript{-9}}\\
    \midrule
    Variance & IPS & 45.156 & 22.051 & 12.052 & 5.491 & 2.712 & 1.423\\
        & IIPS & \textbf{0.024} & \textbf{0.013} & \textbf{0.00637} & \textbf{0.00318} & \textbf{0.00158} & \textbf{7.71\(\times\)10\textsuperscript{-4}}\\
        & RIPS & 3.123 & 1.559 & 0.827 & 0.379 & 0.191 & 0.098\\
        & AIPS & 2.279 & 1.444 & 0.850 & 0.459 & 0.266 & 0.151\\
        & ADR & 2.369 & 1.375 & 0.754 & 0.381 & 0.206 & 0.100\\
        & ADR (oracle) & 0.816 & 0.552 & 0.316 & 0.175 & 0.091 & 0.045\\
    \bottomrule
  \end{tabular}}
\end{table}

\begin{table}[H]
  \centering
  \caption{MSE, squared bias, and variance of all estimators in the ranking-length experiment.}
  \label{tab:all-estimators-ranking}
  \scriptsize
  \setlength{\tabcolsep}{3.5pt}
  \resizebox{0.85\textwidth}{!}{%
  \begin{tabular}{llrrrrrr}
    \toprule
    Metric & Estimator & $L=4$ & $L=6$ & $L=8$ & $L=10$ & $L=12$ & $L=14$\\
    \midrule
    MSE & IPS & 0.039 & 0.548 & 5.493 & 54.177 & 387.848 & 796.968\\
        & IIPS & 2.354 & 4.852 & 7.873 & 10.611 & 14.032 & 18.040\\
        & RIPS & 0.635 & 1.137 & 1.974 & 4.414 & 14.069 & 52.874\\
        & AIPS & 0.169 & 0.690 & 1.510 & 3.239 & 6.115 & 10.581\\
        & ADR & 0.016 & 0.087 & 0.402 & 1.444 & 3.498 & 6.645\\
        & ADR (oracle) & \textbf{0.00574} & \textbf{0.032} & \textbf{0.175} & \textbf{0.551} & \textbf{1.328} & \textbf{2.586}\\
    \midrule
    Squared bias & IPS & $9.29{\times}10^{-7}$ & $1.74{\times}10^{-5}$ & $1.31{\times}10^{-4}$ & \textbf{2.16\(\times\)10\textsuperscript{-5}} & 0.026 & 0.690\\
        & IIPS & 2.352 & 4.850 & 7.870 & 10.607 & 14.027 & 18.035\\
        & RIPS & 0.625 & 1.070 & 1.595 & 1.988 & 2.126 & 2.583\\
        & AIPS & 0.148 & 0.578 & 1.051 & 1.726 & 2.572 & 4.038\\
        & ADR & $8.69{\times}10^{-6}$ & 0.00105 & 0.021 & 0.129 & 0.426 & 1.024\\
        & ADR (oracle) & \textbf{4.54\(\times\)10\textsuperscript{-7}} & \textbf{7.21\(\times\)10\textsuperscript{-6}} & \textbf{3.25\(\times\)10\textsuperscript{-7}} & $3.66{\times}10^{-5}$ & \textbf{1.83\(\times\)10\textsuperscript{-5}} & \textbf{5.87\(\times\)10\textsuperscript{-4}}\\
    \midrule
    Variance & IPS & 0.039 & 0.548 & 5.493 & 54.177 & 387.822 & 796.278\\
        & IIPS & \textbf{0.00199} & \textbf{0.00259} & \textbf{0.00318} & \textbf{0.00383} & \textbf{0.00467} & \textbf{0.00558}\\
        & RIPS & 0.010 & 0.067 & 0.379 & 2.426 & 11.943 & 50.291\\
        & AIPS & 0.021 & 0.112 & 0.459 & 1.513 & 3.543 & 6.543\\
        & ADR & 0.016 & 0.086 & 0.381 & 1.315 & 3.073 & 5.620\\
        & ADR (oracle) & 0.00574 & 0.032 & 0.175 & 0.551 & 1.328 & 2.586\\
    \bottomrule
  \end{tabular}}
\end{table}

In the data-size experiment, the conventional estimators fail to attain
accurate estimates for different reasons. The variance of IPS decreases from
$45.156$ at $n=1{,}000$ to $1.423$ at $n=32{,}000$, but remains large because
IPS weights the entire ranking. IIPS reduces variance from $0.024$ to $0.001$,
but its squared bias remains approximately $7.8$, leaving its MSE nearly
unchanged as the data size grows. RIPS also reduces variance from $3.123$ to
$0.098$, while its squared bias remains between $1.702$ and $1.597$. These
results show that fixed browsing assumptions can reduce variance while leaving
squared bias when they mismatch the actual user behavior. At $n=32{,}000$, ADR
reduces the squared bias of AIPS from $1.205$
to $0.006$ and its variance from $0.151$ to $0.100$. Thus, reductions in both
squared bias and variance contribute to ADR's MSE improvement in the
data-size experiment.

In the ranking-length experiment, the weaknesses of the conventional
estimators become more pronounced as the ranking grows. IPS remains nearly
unbiased, but its variance grows rapidly and reaches $796.278$ at $L=14$.
IIPS reduces variance to $0.006$, but its squared bias increases to $18.035$.
RIPS has squared bias $2.583$ and variance $50.291$. ADR achieves the smallest MSE
among all non-oracle estimators at every ranking length. At $L=14$, replacing
AIPS with ADR reduces squared bias from $4.038$ to $1.024$ and variance from
$6.543$ to $5.620$. ADR therefore improves both error components of AIPS even
for long rankings.

\subsection{Comparison of Reward-Model Information}
\label{app:oracle-ablation}

We compare ADR, ADR (marginal true $q$), and ADR (latent true $q$).
All three estimators use the same estimated user behavior model $\widehat c$
and self-normalized importance weights; they differ only in the reward model.

ADR does not observe the latent user behavior $c_i$ realized for each sample
and uses a reward model $\widehat q_l(x,a)$ learned from the observable
$(x_i,a_i,r_i)$. ADR (marginal true $q$) represents the population target of
an ideally learned reward model based on observable $(x,a)$ and uses
\[
  q_l^{\mathrm{marg}}(x,a)
  =
  \mathbb E_{p(c\mid x)}
  \left[q_l(x,a,c)\right]
\]
This quantity averages the true expected reward over possible user behaviors given
$x$, without using the realized $c_i$. ADR (latent true $q$) instead uses the
true expected reward $q_l(x,a,c_i)$ conditional on the $c_i$ realized for each
sample. It is therefore a stronger reference that uses $c_i$, which is
unavailable in ordinary logged data, for reward prediction.

ADR (marginal true $q$) is not the perfect reward model $q_l(x,a,c)$ used in
the theoretical special case. It averages over $c$ given the observable
$(x,a)$. Hence, its difference from the true expected reward conditional on
the realized $c_i$,
\[
  \Delta_{q,q^{\mathrm{marg}},l}(x,a,c_i)
  =
  q_l^{\mathrm{marg}}(x,a)-q_l(x,a,c_i)
\]
is generally nonzero. In contrast, ADR (latent true $q$) uses
$q_l(x,a,c_i)$, so its reward-model error is zero.

\begin{figure}[H]
  \centering
  \includegraphics[width=\textwidth]{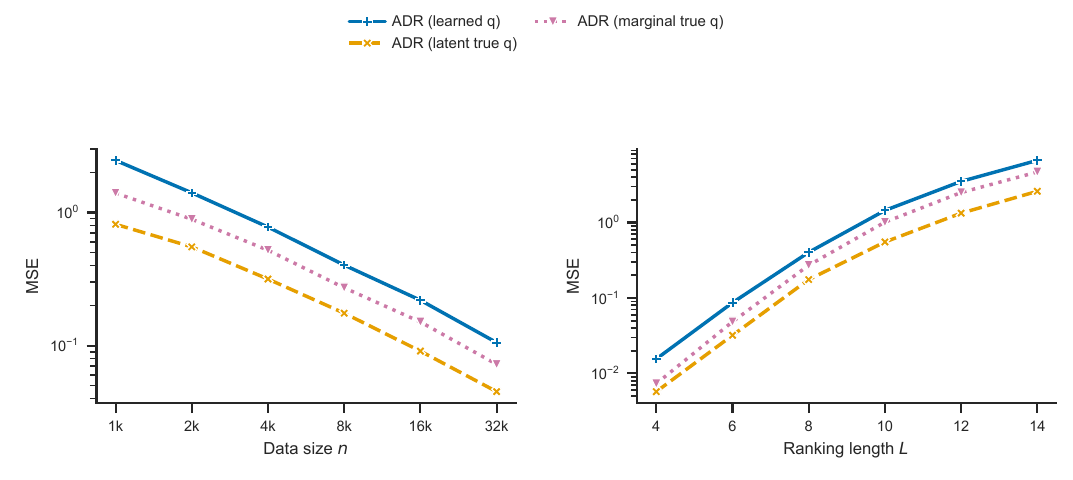}
  \caption{Comparison of reward-model information. The left and right panels show MSE against data size and ranking length, respectively.}
  \Description{Two line charts compare ADR using a learned reward model, a marginalized true reward model, and a true reward model conditional on latent user behavior.}
  \label{fig:appendix-oracle}
\end{figure}

\begin{table}[H]
  \centering
  \caption{MSE of ADR with different reward models at representative data sizes.}
  \label{tab:oracle-ablation-data}
  \begin{tabular}{lrrr}
    \toprule
    $n$ & ADR & Marginal true $q$ & Latent true $q$\\
    \midrule
    $1{,}000$ & 2.465 & 1.399 & 0.816\\
    $32{,}000$ & 0.105 & 0.072 & 0.045\\
    \bottomrule
  \end{tabular}
\end{table}

\begin{table}[H]
  \centering
  \caption{MSE of ADR with different reward models at representative ranking lengths.}
  \label{tab:oracle-ablation-ranking}
  \begin{tabular}{lrrr}
    \toprule
    $L$ & ADR & Marginal true $q$ & Latent true $q$\\
    \midrule
    $4$ & 0.0156 & 0.00732 & 0.00574\\
    $14$ & 6.645 & 4.660 & 2.586\\
    \bottomrule
  \end{tabular}
\end{table}

The purpose of this comparison is to separate the gap between ADR and the
latent-true-$q$ oracle into two sources. The gap from ADR to marginal true $q$
quantifies how far ADR can improve when its reward model is made ideal within
the observable information.
Relative to learned $q$, marginal true $q$ reduces MSE by $31.4\%$ at
$n=32{,}000$ and by $29.9\%$ at $L=14$. Thus, improving the reward model
learned from ordinary logged data can further improve ADR's estimation
accuracy.

The remaining gap from marginal true $q$ to latent true $q$ reflects the
effect of additionally using each sample's realized $c_i$ for reward
prediction. Latent true $q$ reduces MSE relative to marginal true $q$ by
$37.5\%$ at $n=32{,}000$ and by $44.5\%$ at $L=14$. This additional gain
cannot be obtained merely by perfectly predicting the population target
$q_l^{\mathrm{marg}}(x,a)$. This decomposition distinguishes whether the gap
between ADR and the oracle arises from reward-model learning and approximation
error or from the oracle's access to the unobserved $c_i$.

The variance difference in Theorem~1 contains the conditional variance of
the importance-weighted reward-model error. The error remaining under
marginal true $q$ therefore reduces the variance improvement that ADR can
obtain over AIPS. In contrast, latent true $q$ corresponds to the ideal
reward-model condition because its reward-model error is zero. However, all
three estimators in this comparison retain the estimated user behavior model
and self-normalized weights. Thus, this comparison does not reproduce the
exact setting of Theorem~1; it examines how the information available to the
reward model affects estimation accuracy.

\subsection{Effect of Self-Normalization}
\label{app:self-normalization}

The AIPS and ADR estimators defined in the theoretical analysis are standard
sample averages with unnormalized weights $w_{i,l}$. In contrast, the main
experiments use
\[
  \widetilde w_{i,l}
  :=
  \frac{w_{i,l}}
       {\frac1n\sum_{j=1}^n w_{j,l}}
\]
The normalized weights have sample mean one at each position. The theoretical
analysis uses unnormalized ADR because its change-of-measure identity holds
exactly. The main experiments instead apply self-normalization to both AIPS
and ADR to limit the extent to which extremely large weights dominate the
estimate for long rankings and to compare ADR's residual correction under the
same practical conditions.

Because the normalizing term is itself a random variable computed from the
logged data, self-normalization generally loses exact finite-sample
unbiasedness, and the theoretical guarantees in the main text do not directly
apply to the self-normalized estimator. In return, it can reduce variance by
limiting the influence of extreme weights, creating a bias--variance trade-off.
We compare AIPS (raw), AIPS (self-normalized), ADR (raw), and ADR
(self-normalized), all based on the same user behavior model learned from
logged data.

\begin{figure}[H]
  \centering
  \includegraphics[width=0.82\textwidth]{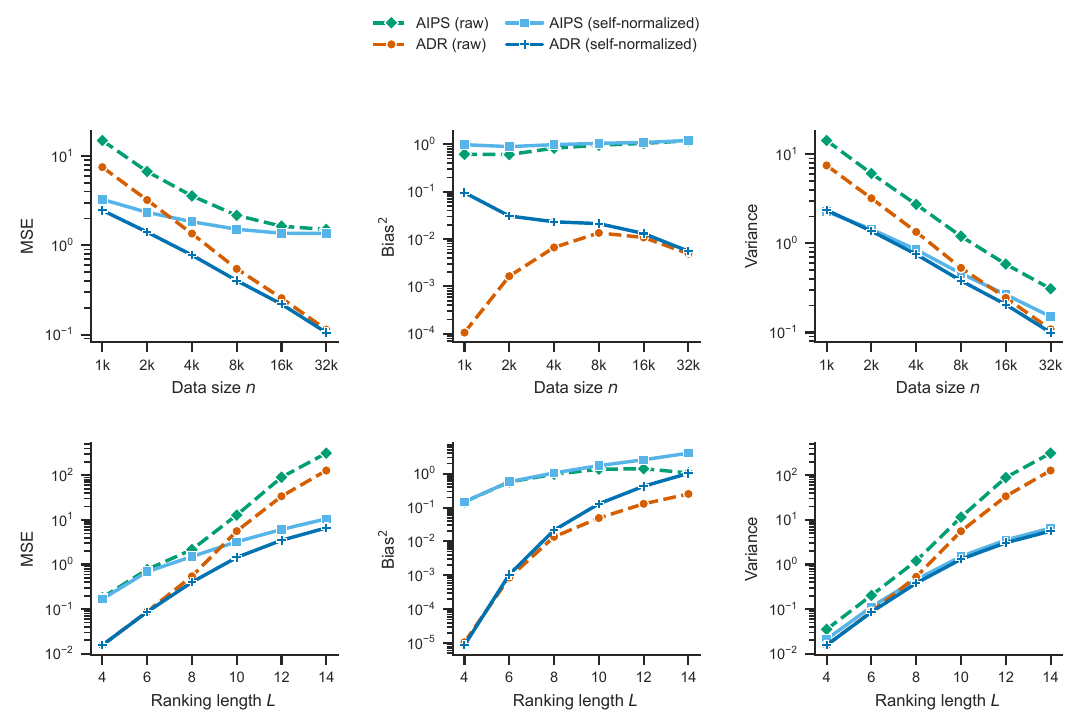}
  \caption{Effect of self-normalization on AIPS and ADR. The top and bottom rows report MSE, squared bias, and variance against data size and ranking length, respectively.}
  \Description{Six line charts compare MSE, squared bias, and variance of raw and self-normalized AIPS and ADR estimators.}
  \label{fig:self-normalization}
\end{figure}

\begin{table}[H]
  \centering
  \caption{MSE, squared bias, and variance before and after self-normalization at representative conditions.}
  \label{tab:self-normalization-representative}
  \small
  \begin{tabular}{llrrr}
    \toprule
    Condition & Estimator & MSE & Bias$^2$ & Variance\\
    \midrule
    $n=32{,}000$ & AIPS (raw) & 1.505 & \textbf{1.195} & 0.310\\
                  & AIPS (self-normalized) & \textbf{1.356} & 1.205 & \textbf{0.151}\\
                  & ADR (raw) & 0.114 & \textbf{0.00491} & 0.109\\
                  & ADR (self-normalized) & \textbf{0.105} & 0.00554 & \textbf{0.0998}\\
    \midrule
    $L=14$ & AIPS (raw) & 314.774 & \textbf{1.049} & 313.725\\
           & AIPS (self-normalized) & 10.581 & 4.038 & \textbf{6.543}\\
           & ADR (raw) & 128.327 & \textbf{0.252} & 128.075\\
           & ADR (self-normalized) & \textbf{6.645} & 1.024 & \textbf{5.620}\\
    \bottomrule
\end{tabular}
\end{table}

\paragraph{Bias--variance trade-off.}
The central result in Table~\ref{tab:self-normalization-representative} is
that self-normalization improves MSE because its reduction in variance
outweighs the increase in squared bias.

\paragraph{Data-size experiment.}
At $n=32{,}000$, self-normalization slightly increases squared bias from
$1.195$ to $1.205$ for AIPS and from $0.00491$ to $0.00554$ for ADR.
In contrast, it reduces variance from $0.310$ to $0.151$ for AIPS
(a $51.3\%$ reduction) and from $0.109$ to $0.0998$ for ADR
(an $8.1\%$ reduction). Consequently, MSE decreases from $1.505$ to $1.356$
for AIPS (a $9.9\%$ improvement) and from $0.114$ to $0.105$ for ADR
(a $7.2\%$ improvement). Thus, self-normalization improves MSE even when
the logged-data size is large, although its effect is smaller than for long
rankings (Table~\ref{tab:self-normalization-representative}).

\paragraph{Ranking-length experiment.}
The effect of self-normalization is more pronounced for long rankings. At
$L=14$, squared bias increases from $1.049$ to $4.038$ for AIPS and from
$0.252$ to $1.024$ for ADR. However, variance decreases from $313.725$ to
$6.543$ for AIPS (a $97.9\%$ reduction) and from $128.075$ to $5.620$ for
ADR (a $95.6\%$ reduction). Because this variance reduction outweighs the
increase in squared bias, MSE decreases from $314.774$ to $10.581$ for AIPS
(a $96.6\%$ improvement) and from $128.327$ to $6.645$ for ADR
(a $94.8\%$ improvement). Thus, self-normalization introduces finite-sample
bias but substantially reduces variance and improves MSE. Although outside
the theoretical guarantees in the main text, it is a practical choice for
evaluating long rankings with finite logged data.

\end{document}